\documentclass[letterpaper, 10 pt, journal, twoside]{IEEEtran}
\ifCLASSINFOpdf
   \usepackage[pdftex]{graphicx}
\else
\fi
\newcommand{\ignore}[1]{}

\newcommand{\norm}[1]{\left\Vert#1\right\Vert} % Norm
\newcommand{\mc}[1]{\mathcal{#1}}

\newcommand{\bma}[1]{\left[\begin{array}{ #1}}
\newcommand{\ema}{\end{array}\right]}

\DeclareMathAlphabet{\mbf}{OT1}{ptm}{b}{n}
\newcommand{\mbs}[1]{{\boldsymbol{#1}}}
\newcommand{\mbfdot}[1]{{\dot{\mbf{#1}}}}

\newcommand{\mbfhat}[1]{{\hat{\mbf{#1}}}}

\def\fdotb{{\raisebox{-0.6ex}{ \kern0.2ex\raisebox{0.8ex}{\tiny $\hspace*{-1ex}\circ$}}}}
\def\fddotb{{\raisebox{-0.6ex}{ \kern0.2ex\raisebox{0.8ex}{\tiny $\hspace*{-1ex}\circ\circ$}}}}

\newcommand{\p}{\partial}
\newcommand{\f}{\frac}

\newcommand{\dee}{\textrm{d}}

\newcommand{\trans}{{\ensuremath{\mathsf{T}}}} % transpose
\newcommand{\utimes}{ {\raisebox{-0.6ex}{ \kern-1.0ex\raisebox{0.6ex}{ \small $\mathsf{v}$}}} } % 
\newcommand{\beq}{\begin{equation}}
\newcommand{\eeq}{\end{equation}}
\newcommand{\bdis}{\begin{displaymath}}
\newcommand{\edis}{\end{displaymath}}
\newcommand{\beqarray}{\begin{eqnarray}}
\newcommand{\eeqarray}{\end{eqnarray}}
\newcommand{\beqarraynn}{\begin{eqnarray*}}
\newcommand{\eeqarraynn}{\end{eqnarray*}}
\newcommand{\balign}{\begin{align}}
\newcommand{\ealign}{\end{align}}
\newcommand{\balignnn}{\begin{align*}}
\newcommand{\ealignnn}{\end{align}}

\renewcommand{\theenumii}{\arabic{enumii}}
\makeatletter
\renewcommand{\p@enumii}{\theenumi.}
\makeatother
\newtheorem{theorem}{Theorem}
\newtheorem{corollary}{Corollary}
\newtheorem{remark}{Remark}

\usepackage{arydshln}
\usepackage{mathtools}

\usepackage{amsmath} % cmex10
\usepackage{amssymb}
\usepackage{float}

\usepackage{cite}

\begin{document}
%auto-ignore
% This is not a standalone latex document. To use this file
% as a cover page on an arXiv upload of a document that is 
% already accepted as some sort of IEEE publication, you must
%
%  1) add the following just after the \begin{document} line
%     of your main paper document
%
%         \input{arxiv-cover-ieee.tex}
%
%  2) and replace the relevant information in the block below.
%
% The relevant information has been parameterized as variables.
% Simply replace the variable values with your stuff and the 
% result should be good.
%
% Make sure to not include this file for ACTUAL submissions to 
% the IEEE. Luckily you can just comment in/out the 
% \input{arxiv-cover-ieee.tex} line.
%
% FYI: The exact citation with formatting can be obtained 
% from your paper's page on IEEE Xplore.
%
%%%%%%%%%%%%%%%%%%%%%%%%%%%%%%%%%%%%%%%%%%%%%%%%%%%%%%%%%%%%%%%
%%%%%%%%%%%%%%%%%%%%%% ADD YOUR INFO HERE %%%%%%%%%%%%%%%%%%%%%
%%%%%%%%%%%%%%%%%%%%%%%%%%%%%%%%%%%%%%%%%%%%%%%%%%%%%%%%%%%%%%%
\def \myJournal {IEEE Robotics and Automation Letters}
\def \myDoi {10.1109/LRA.2021.3067253}
\def \myPaperSiteName {IEEE Xplore}
\def \myPaperSiteLink {https://ieeexplore.ieee.org/document/9381606}
\def \myYear {2021}
\def \myPaperCitation{M. Shalaby, C. C. Cossette, J. R. Forbes and J. Le Ny, ``Relative Position Estimation in Multi-Agent Systems Using Attitude-Coupled Range Measurements,'' in \textit{IEEE Robotics and Automation Letters}, vol. 6, no. 3, pp. 4955 - 4961, July 2021.}

%%%%%%%%%%%%%%%%%%%%%%%%%%%%%%%%%%%%%%%%%%%%%%%%%%%%%%%%%%%%%%%
%%%%%%%%%%%%%%%%%%%%%%%%%%%%%%%%%%%%%%%%%%%%%%%%%%%%%%%%%%%%%%%

\begin{figure*}[t]

\thispagestyle{empty}
\begin{center}
\begin{minipage}{6in}
\centering
This paper has been accepted for publication in \emph{\myJournal}. 
\vspace{1em}

This is the author's version of an article that has, or will be, published in this journal or conference. Changes were, or will be, made to this version by the publisher prior to publication.
\vspace{2em}

\begin{tabular}{rl}
DOI: & \myDoi\\
\myPaperSiteName: & \texttt{\myPaperSiteLink}
\end{tabular}

\vspace{2em}
Please cite this paper as:

\myPaperCitation

\vspace{15cm}
\copyright \myYear \hspace{4pt}IEEE. Personal use of this material is permitted. Permission from IEEE must be obtained for all other uses, in any current or future media, including reprinting/republishing this material for advertising or promotional purposes, creating new collective works, for resale or redistribution to servers or lists, or reuse of any copyrighted component of this work in other works.

\end{minipage}
\end{center}
\end{figure*}
\newpage
\clearpage
\pagenumbering{arabic} 

\title{Relative Position Estimation in Multi-Agent Systems Using Attitude-Coupled Range Measurements}

\author{Mohammed Shalaby$^{1}$, Charles Champagne Cossette$^{1}$, James Richard Forbes$^{1}$, and Jerome Le Ny$^{2}$%
\thanks{Manuscript received: Oct. 15th, 2020; Revised Jan. 15th, 2021; Accepted Feb. 19th, 2021.}%Use only for final RAL version
\thanks{This paper was recommended for publication by Editor M. Ani Hsieh upon evaluation of the Associate Editor and Reviewers' comments.
This work was supported by FRQNT under grant 2018-PR-253646, the William Dawson Scholar program, the NSERC Discovery Grant program, and the CFI JELF program.} %Use only for final RAL version
\thanks{$^{1}$M. Shalaby, C. C. Cossette, and J. R. Forbes are with the department of Mechanical Engineering, McGill University, Montreal, QC H3A 0C3, Canada. {\tt\footnotesize mohammed.shalaby@mail.mcgill.ca, charles.cossette@mail.mcgill.ca, james.richard.forbes@mcgill.ca.}}%
\thanks{$^{2}$J. Le Ny is with the department of Electrical Engineering, Polytechnique Montreal, Montreal, QC H3T 1J4, Canada. {\tt\footnotesize jerome.le-ny@polymtl.ca.}}
\thanks{Digital Object Identifier (DOI): see top of this page.}
}


% Paper headers
\markboth{IEEE Robotics and Automation Letters. Preprint Version. Accepted February, 2021}
{Shalaby \MakeLowercase{\textit{et al.}}: Relative Position Estimation in Multi-Agent Systems Using Attitude-Coupled Range Measurements} 
% Use only for final RAL version



\maketitle

\begin{abstract}
The ability to accurately estimate the position of robotic agents relative to one another, in possibly GPS-denied environments, is crucial to execute collaborative tasks. Inter-agent range measurements are available at a low cost, due to technologies such as ultra-wideband radio. However, the task of three-dimensional relative position estimation using range measurements in multi-agent systems suffers from unobservabilities. This letter presents a sufficient condition for the observability of the relative positions, and satisfies the condition using a simple framework with only range measurements, an accelerometer, a rate gyro, and a magnetometer. The framework has been tested in simulation and in experiments, where 40-50 cm positioning accuracy is achieved using inexpensive off-the-shelf hardware.
\end{abstract}


% \fi
\IEEEpeerreviewmaketitle

\section{Introduction}

\IEEEPARstart{R}{elative} positioning is a key requirement in many multi-robot applications, such as formation control, collision avoidance, and collaborative simultaneous localization and mapping (SLAM).
Range measurements provide a means to acquire inter-robot distance information. Range sensors are particularly attractive as they are generally inexpensive, light, computationally simple, and can be used in GPS-denied environments. This reduces the physical and computational requirements of the agents, and allows the deployment of large swarms of small, inexpensive robots in indoor or underground environments.

Estimating three-dimensional relative positions using range measurements is a non-trivial task, as there is an infinite number of possible solutions given a single range measurement. This unobservability arises from the fact that any group of agents can be collectively rotated in three-dimensional space while maintaining constant inter-agent distances, as no bearing information is available. There exist a multitude of approaches that attempt to fuse the range measurements with additional information to achieve an observable problem.

\begin{figure}
    \centering
    \includegraphics[width=0.9\columnwidth]{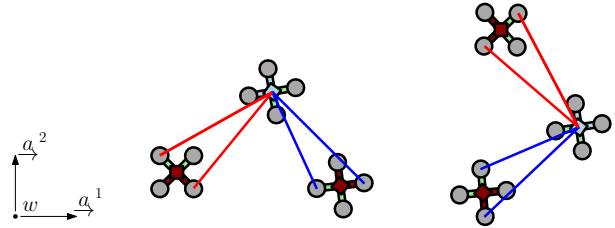}
    \caption{A schematic of the two-tag agent framework, where the red agents possess two ranging \emph{tags}, and are referred to as two-tag agents, and the blue agent is a single-tag agent.}
    \label{fig:twoTagAgents}
\end{figure}

Most indoor localization approaches traditionally assume the existence of an infrastructure of 4 or more \emph{anchors} with known positions \cite{Mueller2015a, Cano2019, Mai2018}. However, the problem of relative localization has also been addressed in the presence of a single anchor. In \cite{Cao2020}, only a single anchor is used for position tracking, with a constant velocity assumption, while a range-based SLAM approach is utilized for relative localization in \cite{Lourenco2015} and \cite{Cao2020SLAM}. Other approaches that eliminate the need for anchors in two-dimensional relative position estimation include \cite{Nguyen2019} and \cite{VanderHelm2018}, which assume displacement measurements are available using an optical flow sensor. In \cite{cossette2020}, a sliding window filter is able to estimate the three-dimensional relative position between two agents using just single-range measurements and 9-axis inertial measurement units (IMUs). All these single range-based localization approaches usually require persistent relative motion between the anchor and the agent \cite{Cao2020,Lourenco2015,Cao2020SLAM} or the two agents \cite{Nguyen2019,VanderHelm2018,cossette2020}, as outlined in \cite{Batista2011a}. In the presence of many agents, \cite{Williams2015} and \cite{Heintzman2020} show how the nonlinear observability matrix associated with a two-dimensional relative localization problem is dependent on both the rigidity matrix and the relative motion of the agents. Alternative approaches include the implementation of a particle filter when only an IMU and range measurements are available, as in \cite{Liu2017}, which is computationally expensive.

More recently, the idea of using multi-tag agents, as shown in Fig. \ref{fig:twoTagAgents}, has been proposed. In \cite{Richardson2010}, ultra-wideband (UWB) range sensors are used for relative positioning of trucks fitted with two tags, and in \cite{Guler2019}, an agent is equipped with three tags. Both these methods extract two-dimensional relative position information in the body frame of the computing agent. In \cite{Hepp2016a}, the users were capable of tracking a person in two dimensions using a special ranging protocol with 4 tags on an agent, while in \cite{Nguyen2019Platform}, multiple tags on a moving platform allow an agent to approach using range and relative displacement measurements, and land using vision and range measurements. Lastly, in \cite{Nguyen2018}, the use of two two-tag agents is coupled with an altimeter and optical flow velocity measurements for relative localization, and the results are validated in an experiment with limited motion. The main limitation of the results in \cite{Richardson2010, Guler2019, Hepp2016a, Nguyen2019Platform, Nguyen2018} is that the analysis is mainly restricted to only two agents, and no observability analysis is considered.

The contributions of this letter are threefold. The first contribution is a rigidity theory-based observability analysis for any number of agents, where a sufficient condition that is independent of the relative motion of the agents is derived for the observability of the three-dimensional relative positions when only range measurements are available. This motivates the second contribution, which is an extension of the two-tag framework to multi-agent systems that allows the estimation of three-dimensional relative positions using just the range measurements and a low-cost 9-axis IMU, for any number of agents, provided at least two two-tag agents are present. This framework is shown to be \emph{instantaneously locally observable}, as per the sufficient condition, which means that the system is locally observable at any given point in time without any specific trajectory requirements. Lastly, the performance of this framework in simulation and in experiment using multiple agents equipped with inexpensive sensors is presented.

The remainder of this letter is organized as follows. Graph theoretic and observability concepts are reviewed in Section \ref{sec:rigidy_observability}. A sufficient condition for observability of a three-dimensional relative localization problem is addressed in Section \ref{sec:suff_cond}. The two-tag framework is discussed in Section \ref{sec:proposed} and is validated in simulation and in experiment in Sections \ref{sec:sim} and \ref{sec:exp}, respectively.

\section{Rigidity Theory and Instantaneous Local Observability} \label{sec:rigidy_observability}

This letter uses graph theory as one of the fundamental tools for observability analysis.
Consider an \emph{undirected graph} $G = (\mathcal{V}, \mathcal{E})$ consisting of a set of $n$ vertices and $m$ edges, representing the $n$ tags and $m$ distance measurements, respectively. As such, let $\mbf{r}_a^{p_i p_1} \in \mathbb{R}^3$ represent the unknown location of tag $i$ relative to tag 1, resolved in the $3$-dimensional reference frame $\mathcal{F}_a$. Additionally, let $y_{ij}(\mbf{r}_a^{p_i p_1}, \mbf{r}_a^{p_j p_1}) \in \mathbb{R}$, $(i,j) \in \mathcal{E}$ be the range measurement between tags $i$ and $j$, where $(i,i) \notin \mathcal{E}$, $\forall i \in \mathcal{V}$. All graphs defined in this letter are assumed to have this property. Define a column matrix 
\begin{equation}
\mbf{x} \triangleq \left[ \left(\mbf{r}_a^{p_2 p_1}\right)^\trans \hspace{3pt} \cdots \hspace{3pt} \left( \mbf{r}_a^{p_n p_1} \right)^\trans \right]^\trans \in \mathbb{R}^{3(n-1)} \label{eq:x}
\end{equation}
of the $n-1$ relative positions, and a column matrix $\mbf{y}(\mbf{x}) \in \mathbb{R}^m$ of all the known measurements $y_{ij}$.

Parametrize the vector space spanned by the system state by a new arbitrary variable $t$. To analyze the behaviour of the measurements $\mbf{y}(\mbf{x}(t))$ for any infinitesimal change in the states $\mbf{x}(t)$ with respect to $t$, the derivative
\begin{equation}
    \label{eq:rigidity_derivation}
    \f{\dee \mbf{y}(\mbf{x}(t))}{\dee t} = \f{\p \mbf{y} (\mbf{x}(t))}{\p \mbf{x}} \f{\dee \mbf{x}(t)}{\dee t} \triangleq \mbf{R} \mbfdot{x}(t)
\end{equation}
is computed, where $\mbf{R} \in \mathbb{R}^{m \times 3(n-1)}$ is the \emph{rigidity matrix}. 

Traditionally, rigidity theory in $\mathbb{R}^3$ is concerned with the notion of \emph{infinitesimal rigidity} by achieving $\operatorname{rank} \hspace{1pt} \mbf{R} = 3n-6$ for $n$ absolute position states, where the 6 degrees of freedom are associated with the translations and rotations of the graph as a whole \cite{asimow1979}, as shown in Fig. \ref{fig:uwbTransRot}. Additionally, when dealing with $n-1$ three-dimensional relative states as in \eqref{eq:rigidity_derivation}, infinitesimal rigidity is also achieved when $\operatorname{rank} \hspace{1pt} \mbf{R} = 3(n-1) - 3 = 3n-6$, where the 3 degrees of freedom are associated with the rotations of the graph as a whole.

\begin{figure}
    \centering
    \includegraphics[width=0.9\columnwidth]{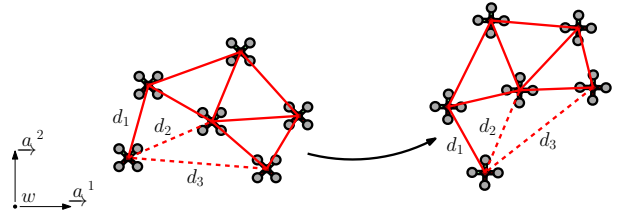}
    \caption{An example of a rigid graph. However, when edges $d_2$ and $d_3$ are removed, the graph becomes \emph{flexible}, as edge $d_1$ can rotate about its upper vertex without breaking any constraints. Additionally, both the rigid and the flexible graphs can translate and rotate without breaking any constraints.}
    \label{fig:uwbTransRot}
\end{figure}

When addressing relative position states, the only non-trivial solutions to $\mbf{R} \mbfdot{x} = \mbf{0}$ of an infinitesimally rigid graph are of the form
\begin{equation}
    \f{\dee}{\dee t}\mbf{r}_a^{p_i p_1} = \mbs{\omega}^\times \mbf{r}_a^{p_i p_1} = - \left(\mbf{r}_a^{p_i p_1}\right)^\times \mbs{\omega}, \quad \forall i \in \mathcal{V} \backslash \{1\}, \label{eq:trans_theorem}
\end{equation}
where $\mbs{\omega}$ denotes a common overall angular velocity of the graph and $(\cdot)^\times$ denotes the skew-symmetric cross product matrix operator in $\mathbb{R}^3$. Without loss of generality, the linearly independent canonical basis vectors $\mbf{e}_i$ of $\mathbb{R}^3$ are chosen as a basis for $\mbs{\omega}$. Therefore, from \eqref{eq:trans_theorem}, the null space of the rigidity matrix of an infinitesimally rigid graph is 
\begin{equation}
    \operatorname{null} (\mbf{R}) = \operatorname{span} \left\{ \mbf{v}_1, \mbf{v}_2, \mbf{v}_3 \right\} \label{eq:rigid_null},
\end{equation}
where
\begin{equation}
    \mbf{v}_i \triangleq \left(\begin{array}{c}
        - \left(\mbf{r}_a^{p_2 p_1}\right)^\times \mbf{e}_i \\
        \vdots \\
        - \left(\mbf{r}_a^{p_n p_1}\right)^\times \mbf{e}_i
    \end{array}\right) \in \mathbb{R}^{3(n-1)}.
\end{equation}
Note that the invariance of the measurements due to common translational motion is a trivial solution when dealing with relative states, since the relative position states are also invariant to common translational motion, meaning $\mbfdot{x}=\mbf{0}$.

A more stringent condition as compared to infinitesimal rigidity is instantaneous local observability, which requires that there is no local trajectory of the states $\mbf{x}$ at any instant in time, excluding the trivial trajectory $\mbfdot{x}=\mbf{0}$, that results in no change in the measurements \cite[Section 6.1]{Hermes1968}. Consequently, as per \eqref{eq:rigidity_derivation}, a system of $n$ agents consisting of $n-1$ relative position states is instantaneously locally observable if $\mbf{R}$ is full rank, that is if $\operatorname{rank} \hspace{1pt} \mbf{R} = 3(n-1)$.

The aim of this work is to disambiguate the aforementioned 3 degrees of freedom corresponding to rotations of the graph as a whole, thus achieving instantaneous local observability for the relative localization problem. In the remainder of this letter, the term \emph{local observability} is used to refer to instantaneous local observability for conciseness.

\section{Sufficient Condition for Local Observability} \label{sec:suff_cond}

Consider a group of $n > 3$ ranging tags navigating 3-dimensional space. The $n-1$ relative position vectors $\mbf{r}_a^{p_i p_1}$, $i = 2,\ldots,n$, are considered. The $1^\text{st}$ tag takes the role of the arbitrary reference point, and is referred to as the \emph{reference tag}. The position between any two tags can be computed using these $n-1$ position vectors relative to the reference tag. Additionally, there is no loss of generality in assuming the $1^\text{st}$ tag as the reference tag, since in practice any tag can be set as the reference tag.

Let $G = (\mathcal{V},\mathcal{E})$ be an infinitesimally rigid undirected graph representing the interconnection topology of the sensor network consisting of the $n$ tags, where the edges represent the distance measurements between pairs of tags. As is, the range measurements are invariant to translations and rotations of the group of tags as a whole, as shown in Fig. \ref{fig:uwbTransRot}, while the relative position states are invariant to the translations only.

\begin{theorem} \label{th:suff_cond}
Consider an infinitesimally rigid undirected graph $G(\mathcal{V},\mathcal{E})$ and its rigidity matrix $\mbf{R}$, where the state vector consists of $n-1$ relative position vectors $\mbf{r}_a^{p_i p_1}$, ${\forall i \in \mathcal{V} \backslash \{1\}}$, and the edges represent range measurements. The graph $\tilde{G}(\mathcal{V}, \tilde{\mathcal{E}})$ constructed from $G$ with two extra edges representing the direct measurement of two linearly independent relative position vectors ${\mbf{r}_a^{p_j p_1}, \mbf{r}_a^{p_\ell p_1} \in \mathbb{R}^3 \backslash \{\mbf{0}\}}$, $j,\ell \in \mathcal{V}\backslash \{1\}$ corresponds to a locally observable system.
\end{theorem}

\begin{proof}
Given the infinitesimal rigidity assumption on $G$, the null space of the rigidity matrix $\mbf{R}$ is as defined in \eqref{eq:rigid_null}. 
Local observability as discussed in Section~\ref{sec:rigidy_observability} requires that $\mbf{R}\mbfdot{x} = \mbf{0}$ if and only if $\mbfdot{x} = \mbf{0}$. Therefore, new knowledge should modify $\mbf{R}$ to generate a new rigidity matrix $\tilde{\mbf{R}}$, such that
\begin{equation}
    \tilde{\mbf{R}} \mbf{z} \neq \mbf{0}, \quad \forall\mbf{z} \in \operatorname{null} \mbf{R} \backslash \{\mbf{0}\}. \label{eq:required_inequality}
\end{equation}
Consider the system corresponding to the graph $\tilde{G}$, where two relative position vectors $\mbf{r}_a^{p_j p_1}$ and $\mbf{r}_a^{p_\ell p_1}$ are measured. The rigidity matrix $\tilde{\mbf{R}}$ is then of the form
\begin{equation}
    \tilde{\mbf{R}} = \left[ \begin{array}{cc}
        \mbf{R}_1^\trans & \mbf{R}^\trans
    \end{array} \right]^\trans,
\end{equation}
where $\mbf{R}_1 \in \mathbb{R}^{6 \times 3(n-1)}$ is the permutation matrix that extracts the measured relative positions $\mbf{r}^{p_j p_1}_a, \; \mbf{r}^{p_\ell p_1}_a$ from $\mbf{x}$, and $\mbf{x}$ is defined as per \eqref{eq:x}. That is,
\begin{equation}
     \left[ \begin{array}{c}
         \mbf{r}_a^{p_j p_1} \\
         \mbf{r}_a^{p_\ell p_1}
    \end{array} \right] = \mbf{R}_1 \mbf{x}.
\end{equation} 
Given that $\mbf{R}\mbf{z} = \mbf{0}$ by the definition of $\mbf{z}$, it is sufficient to show that $\mbf{R}_1 \mbf{z} \neq \mbf{0}$ to achieve \eqref{eq:required_inequality}. This can be rewritten as
\begin{equation}
    \mbf{R}_1 \left[ \begin{array}{ccc}
        \mbf{v}_1& \mbf{v}_2 & \mbf{v}_3 
    \end{array} \right] \mbf{a} \neq \mbf{0}, \quad \mbf{a} = \left[ \begin{array}{c}
        a_1  \\
        a_2 \\
        a_3
    \end{array} \right] \neq \mbf{0}, \label{eq:z_expanded}
\end{equation}
where $a_i$ represent arbitrary scalar parameters, since $\mbf{z}$ is a linear combination of the vectors $\mbf{v}_i$ as shown in \eqref{eq:rigid_null}. By replacing the matrix
\begin{equation}
    \mbf{R}_1 \left[ \begin{array}{ccc}
        \mbf{v}_1& \mbf{v}_2 & \mbf{v}_3 
    \end{array} \right] = \left[ \begin{array}{c}
        - \left(\mbf{r}_a^{p_j p_1}\right)^\times \\
        - \left(\mbf{r}_a^{p_\ell p_1}\right)^\times
    \end{array} \right] \label{eq:V_matrix}
\end{equation}
into \eqref{eq:z_expanded} and assuming that $\mbf{r}_a^{p_j p_1}, \mbf{r}_a^{p_\ell p_1} \neq \mbf{0}$, the expression in \eqref{eq:z_expanded} does not hold if and only if $\left(\mbf{r}_a^{p_j p_1}\right)^\times \mbf{a} =\mbf{0}$ and $\left(\mbf{r}_a^{p_\ell p_1}\right)^\times \mbf{a} =\mbf{0}$, which, due to a property of the cross product, necessitates that the vectors $\mbf{r}_a^{p_j p_1}$, $\mbf{r}_a^{p_\ell p_1}$, and $\mbf{a}$ be collinear. However, $\mbf{r}_a^{p_j p_1}$ and $\mbf{r}_a^{p_\ell p_1}$ are linearly independent and are non-zero by assumption.
Therefore, the system represented by the graph $\tilde{G}$ is locally observable.
\end{proof}

When a system is locally observable, all the relative position vectors are locally unique given the known measurements. Possible approaches to satisfying the minimum knowledge requirement specified by Theorem \ref{th:suff_cond} in multi-agent localization is to fit a small subset of the agents with a GPS or a stereo camera, in addition to the ranging tags. The GPS extracts the relative position information by subtracting the absolute position information, and stereo cameras can extract relative position information using attitude estimates and depth perception. To satisfy the linear independence assumption, agents fitted with a GPS must not lie in a straight line, and the relative position vectors between agents fitted with stereo cameras and the agents they detect must not be all collinear. An alternative approach that does not require additional hardware is discussed in the next section.

\section{Attitude-Coupled Range Measurements} \label{sec:proposed}

\subsection{Overview}

The problem of three-dimensional swarm navigation using an IMU and range measurements is discussed in this section. In the subsequent analysis, an IMU is assumed to consist of accelerometers, gyroscopes, and magnetometers. Typically, the attitude of each agent is observable using the IMU data, as is the case when using an attitude and heading reference system (AHRS) to estimate attitude \cite{farrell2008}, but the relative positions require further measurements, such as a GPS as discussed in the end of Section \ref{sec:suff_cond}. The standard way of utilizing range measurements is usually invariant to each agent's attitude, and thus having access to an agent's attitude provides no additional information regarding its instantaneous position. 

This section discusses an approach to couple the range measurements with the attitude estimates to satisfy the conditions of Theorem~\ref{th:suff_cond}. This allows the integration of an IMU with range measurements for attitude and relative position estimation of a swarm of robots. A key component of this approach is the use of \emph{two-tag agents}, which are agents equipped with two non-collocated ranging tags. The notation of $p_{i,j}$ is used for the point in space of the $j^\text{th}$ tag of Agent~$i$.

\subsection{Two-Tag Agents} \label{subsec:two_tags}

Rather than having one ranging tag on each agent, consider two two-tag agents as shown in Fig. \ref{fig:tt_agents}, while the remaining agents are single-tag agents. A single-tag agent is a conventional agent with only one ranging tag. The range measurement between the $p_{1,i}$ tag and the $p_{2,j}$ tag is then
\begin{align}
    y_{p_{1,i}p_{2,j}} = \norm{\mbf{r}_a^{z_2 z_1} + \mbf{C}_{ab_2} \mbf{r}_{b_2}^{p_{2,j} z_2} - \mbf{C}_{ab_1} \mbf{r}_{b_1}^{p_{1,i} z_1}} , \label{eq:tt_measurement}
\end{align}
where $\mbf{C}_{ab_i} \in SO(3)$ is the direction cosine matrix (DCM) representing the rotation from the body reference frame $\mathcal{F}_{b_i}$ to the absolute frame $\mathcal{F}_{a}$, $\mbf{r}_{b_2}^{p_{2,j} z_2}$ represents the known position of the $j^\text{th}$ tag of Agent 2 relative to the IMU at $z_2$, in the Agent~2 body frame $\mc{F}_{b_2}$, and $\mbf{r}_{b_1}^{p_{1,i} z_1}$ represents the known position of the $i^\text{th}$ tag of Agent 1 relative to the IMU at $z_1$, in the Agent~1 body frame $\mc{F}_{b_1}$. This shows the consequent coupling of the range measurements and the attitude of the agents when considering two-tag agents. The subsequent corollary follows from Theorem~\ref{th:suff_cond}, where it is assumed that for the $i^\text{th}$ two-tag agent the relative position vector $\mbf{r}_{b_i}^{p_{i,2} p_{i,1}} \in \mathbb{R}^3 \backslash \{\mbf{0}\}$ is known, being the position of the second tag relative to the first tag in the body frame $\mathcal{F}_{b_i}$.
\begin{corollary} \label{cor:tt_agents}
    Consider a swarm of $n_t \geq 2$ two-tag agents, each with known attitude, and $n_s$ single-tag agents. Assume that the undirected graph composed of the $2n_t+n_s$ vertices and the range measurements between the tags is infinitesimally rigid. Given that there are at least two two-tag agents $j$ and $\ell$ where $\mbf{r}_a^{p_{j,2}p_{j,1}}$ and $\mbf{r}_a^{p_{\ell,2}p_{\ell,1}}$ are linearly independent, the underlying system representing the relative localization problem is locally observable. 
\end{corollary}
\begin{proof}
    When the attitude of any two-tag agent $i$ is known, the vector $\mbf{r}_a^{p_{i,2} p_{i,1}}$ resembling the relative position vector between the two tags of agent $i$ in the absolute frame $\mathcal{F}_a$ is found through
    \begin{equation}
        \mbf{r}_a^{p_{i,2} p_{i,1}} = \mbf{C}_{ab_i} \mbf{r}_b^{p_{i,2} p_{i,1}}.
    \end{equation}
    Consequently, since at least two linearly independent relative position vectors are known in $\mathcal{F}_a$, the conditions of Theorem \ref{th:suff_cond} are satisfied. Hence, the system is locally observable.
\end{proof}

\begin{figure}
    \centering
    \includegraphics[trim=0.1cm 0.1cm 0 0, clip=true, width=0.85\columnwidth]{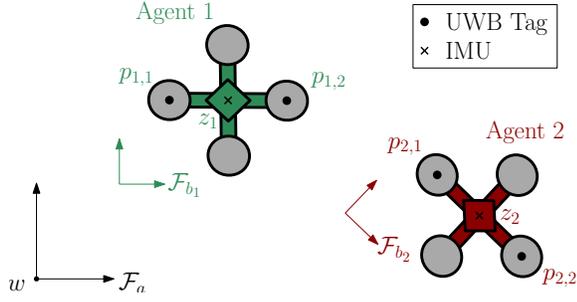}
    \caption{Two two-tag agents, with each agent having two ranging tags at $p_{1,1}$, $p_{1,2}$ and $p_{2,1}$, $p_{2,2}$, respectively, and an IMU at $z_1$ and $z_2$, respectively. The absolute frame is $\mathcal{F}_a$, and the body frame of the $i^\text{th}$ agent is $\mathcal{F}_{b_i}$.}
    \label{fig:tt_agents}
\end{figure}

\begin{remark}
Whenever a swarm of agents equipped with IMUs includes at least two two-tag agents, the problem of finding the relative position of the agents becomes observable using just range measurements within the swarm. However, this also assumes that the known relative position vectors are linearly independent. Therefore, if the two two-tag agents orient themselves such that the known relative position vectors are parallel, the system becomes unobservable. The designer must therefore place the tags in a strategic way to minimize the possibility of these occurrences based on the application. For example, quadcopters rarely go from level flight to a $90^\circ$ pitch or roll orientation, and by placing the two-tags vertically on one agent and horizontally on the other, it is unlikely that the system becomes unobservable. This issue is also mitigated when using more than two two-tag agents, or by fitting more than two tags on an agent, which however adds hardware and congestion on the UWB communication space.
\end{remark}

Note that the rigidity matrix only addresses whether or not ambiguities arise due to the graph being continuously deformable. Two types of ambiguities not considered by the notion of local observability are discontinuous flex ambiguities and flip ambiguities, as discussed in \cite{Hendrickson1992, Moore2004}. For the two-tag agents framework, the user must be aware of the possible occurrence of such ambiguities in the presence of attitude uncertainty, as shown in Figs. \ref{fig:flex_ambig_att_uncertainty} and~\ref{fig:flip_ambig_att_uncertainty}.

\begin{figure}
	\centering
	\begin{minipage}{0.35\columnwidth}
		\centering
        \includegraphics[width=0.9\textwidth]{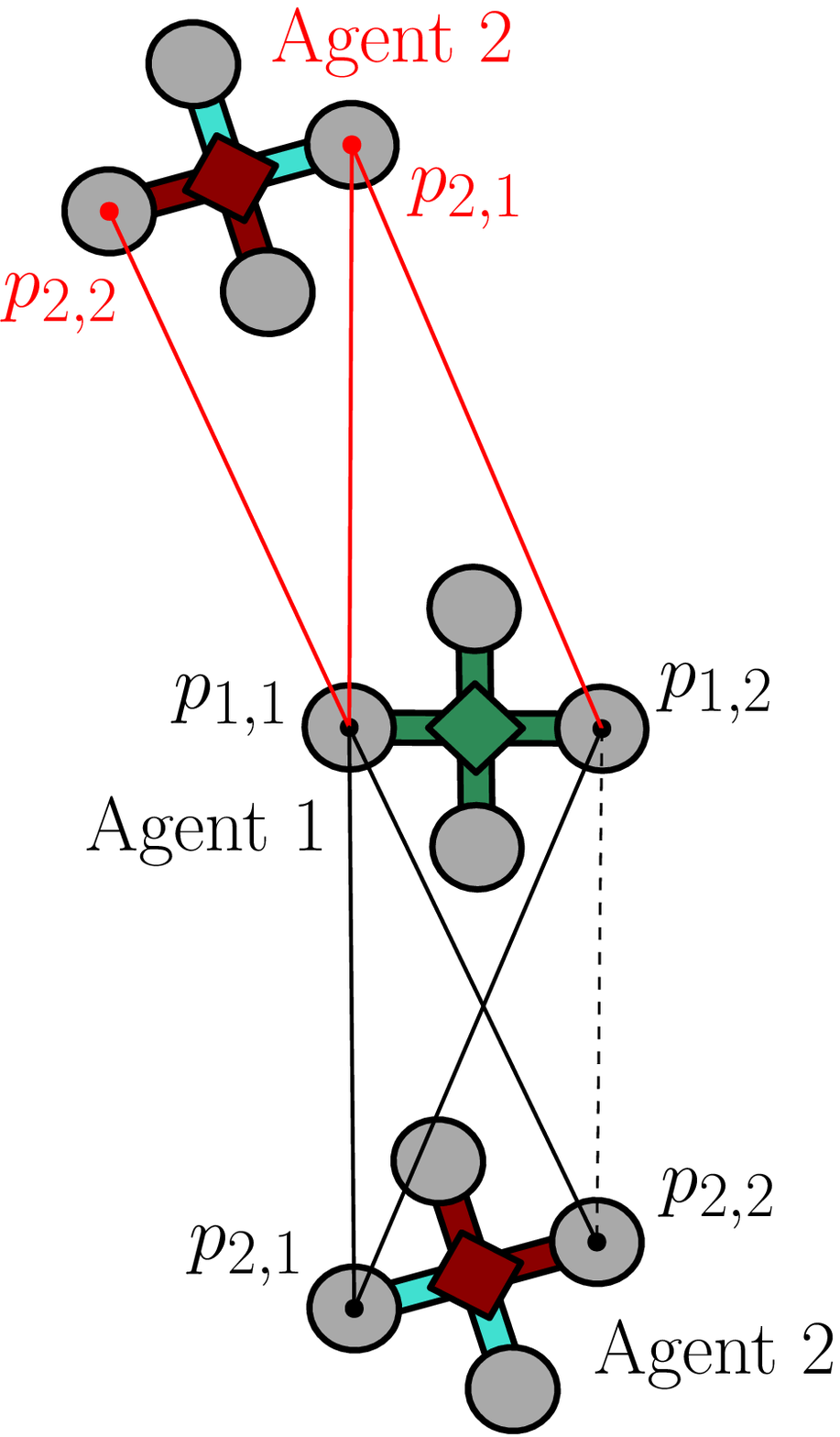}
        \caption{A flex ambiguity with two two-tag agents in $\mathbb{R}^2$. The flex ambiguity does not exist if the dashed edge exists, which might otherwise be assumed redundant in $\mathbb{R}^2$.}
        \label{fig:flex_ambig_att_uncertainty}
	\end{minipage}\hspace{10pt}%
	\begin{minipage}{0.59\columnwidth}
		\centering
        \includegraphics[width=0.99\textwidth]{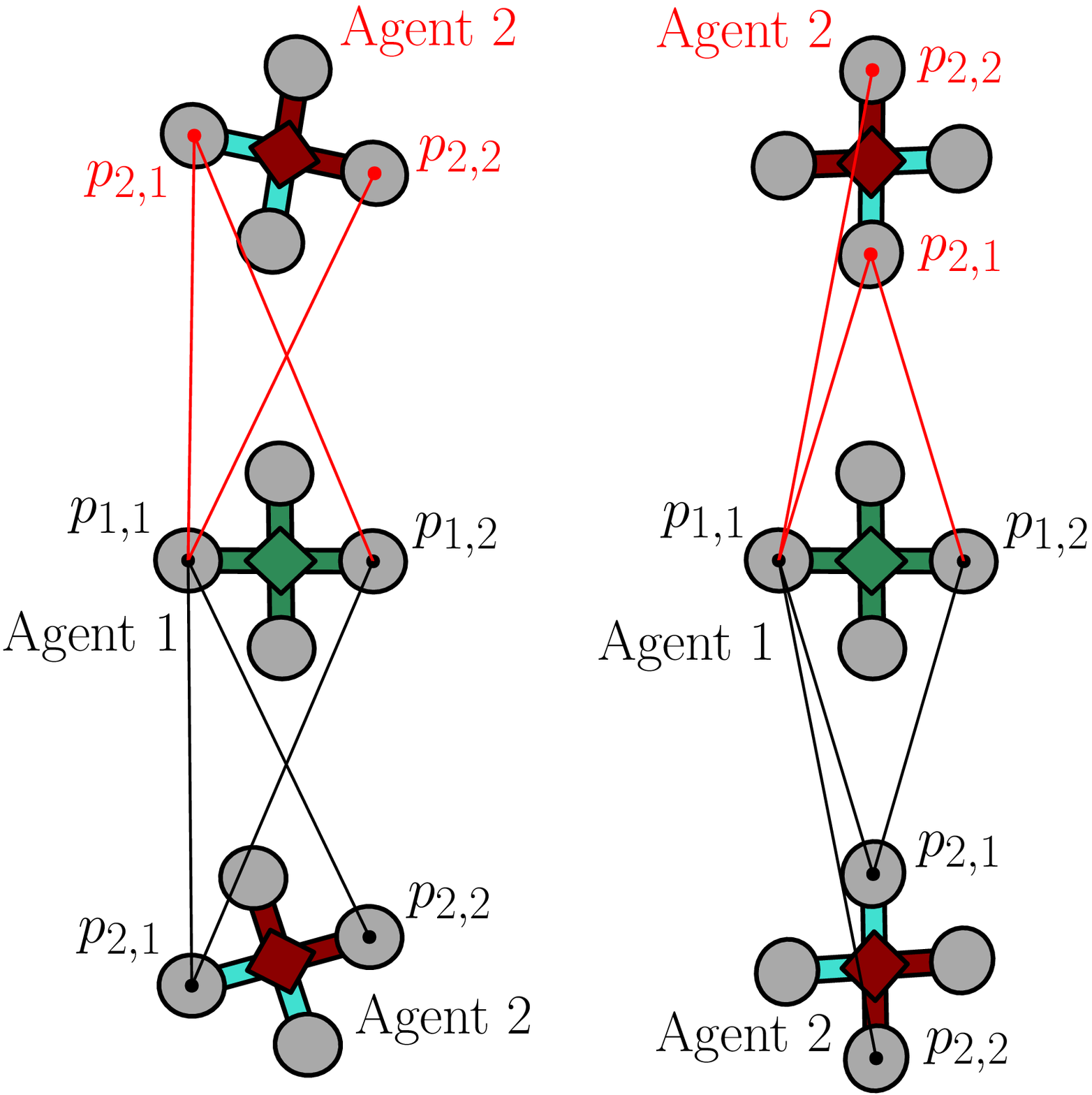}
        \caption{A flip ambiguity with two two-tag agents in $\mathbb{R}^2$. The more parallel the two agents
        are, the more likely that attitude uncertainty results in a flip ambiguity, since the attitude error resulting from the flip ambiguity is smaller.}
        \label{fig:flip_ambig_att_uncertainty}
	\end{minipage}
\end{figure}

\subsection{Relative Position and Attitude Estimator} \label{subsec:models}

\begin{figure*}
	\centering
	\begin{minipage}{0.4\textwidth}
		\centering
        \includegraphics[width=0.97\textwidth, ]{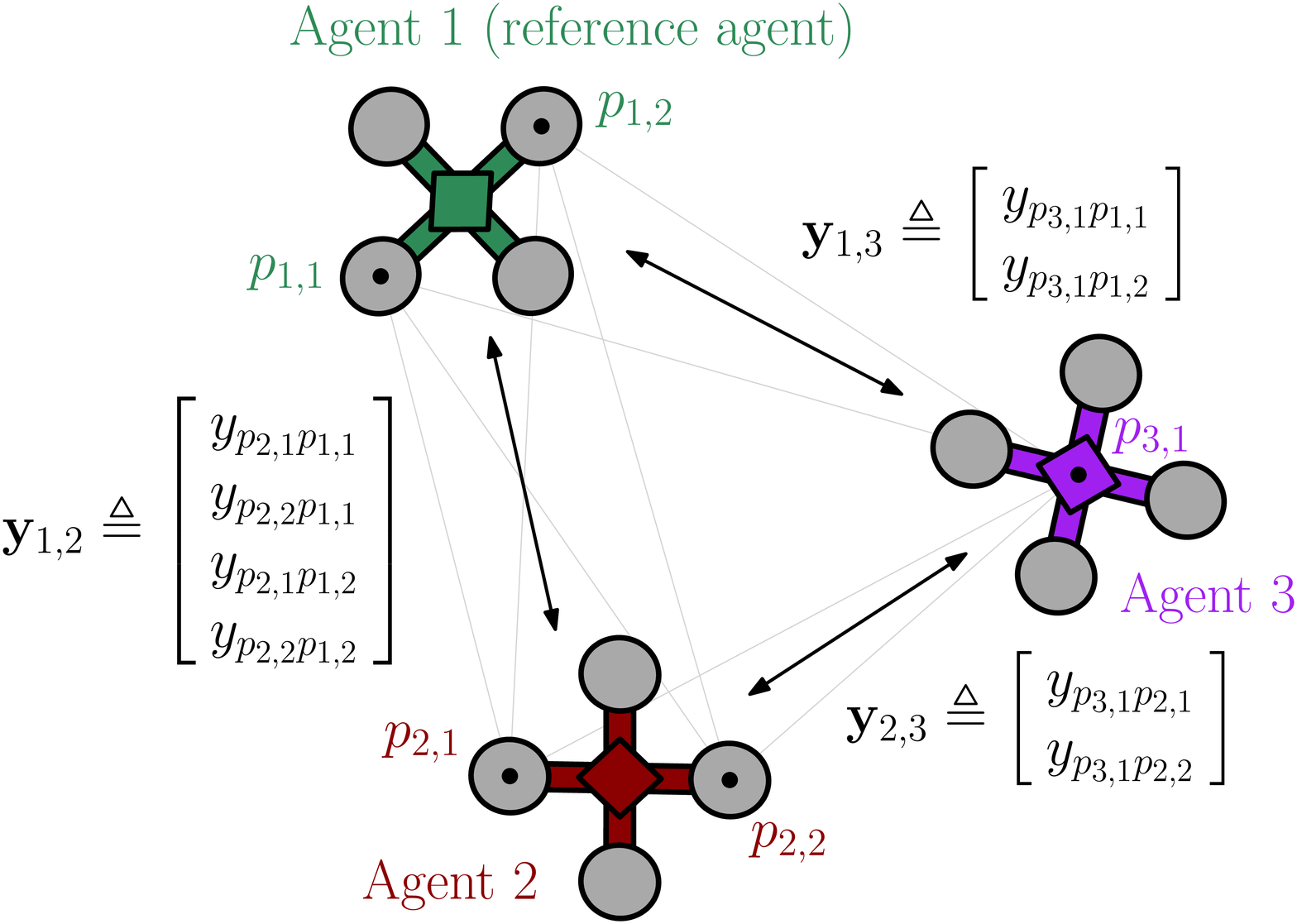}
	\end{minipage} \hspace{5pt} $\Longrightarrow$ \hspace{10pt}
	\begin{minipage}{0.4\textwidth}
		\centering
        \includegraphics[width=0.97\textwidth]{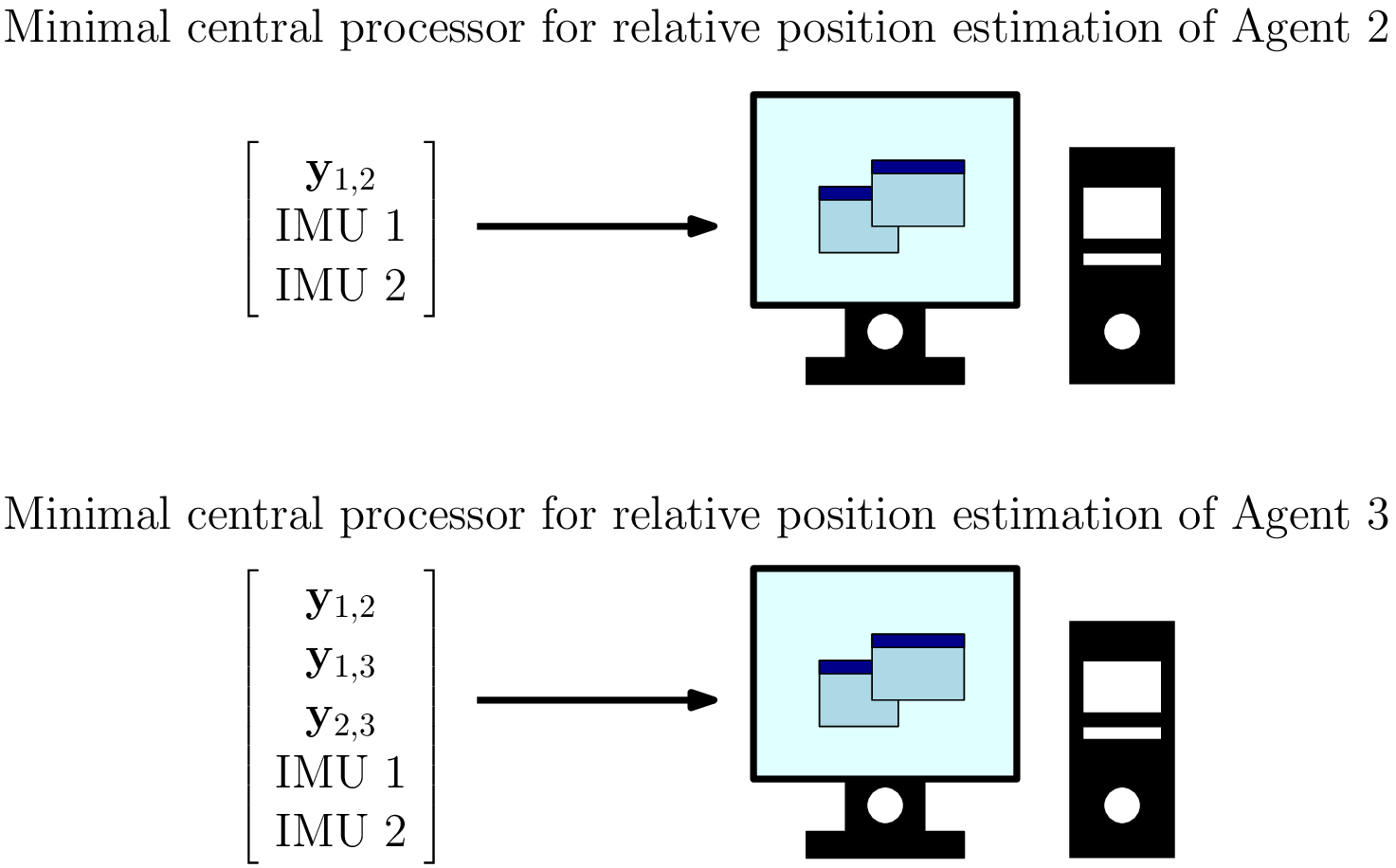}
	\end{minipage}
	\caption{A schematic representing the minimum information required for relative position estimation of two-tag Agent 2 and single-tag Agent 3 relative to the two-tag reference agent Agent 1, in a centralized framework. The tag locations are shown using the black markers, and the light grey lines in the background represent all the range measurements between the ranging tags. The notation $\mbf{y}_{i,j}$ is used to denote all the range measurements between agents $i$ and $j$. As per Corollary \ref{cor:tt_agents}, the relative position vectors $r_{b_1}^{p_{1,2}p_{1,1}}$ and $r_{b_2}^{p_{2,2}p_{2,1}}$ are assumed to be known. In the minimal central processors, note that the IMU of agents 1 and 2 are required as the relative localization problem is dependant on estimating their attitude, while the IMU of Agent 3 is neglected as the attitude of Agent 3 is irrelevant. However, filtering approaches do require the IMU of Agent 3 as well for prediction.}
        \label{fig:decentralized}
\end{figure*}

% \begin{figure}
%     \centering
%     \includegraphics[width=\columnwidth]{figs/misc/decentralized.eps}
%     \caption{A schematic of a reference two-tag Agent 1 sharing its estimate of $\mbf{r}_a^{p_{1,2} p_{1,1}}$ with a two-tag Agent 2 and a single-tag Agent 3. Additionally, Agent 2 shares its estimate of $\mbf{r}_a^{p_{2,2} p_{2,1}}$ with Agent 3.}
%     \label{fig:decentralized}
% \end{figure}

Corollary \ref{cor:tt_agents} requires the presence of at least two-tag agents to achieve local observability. Therefore, any single-tag agent needs to communicate with at least two two-tag agents, and any two-tag agent needs to communicate with at least one other two-tag agent. Additionally, each two-tag agent $i$ must compute $\mbfhat{r}_a^{p_{i,2} p_{i,1}}$ based on its attitude estimate $\mbfhat{C}_{ab_i}$ and the known vector $\mbf{r}_{b_i}^{p_{i,2} p_{i,1}}$. Therefore, by satisfying these minimum ranging conditions, and with IMU measurements on two-tag agents for attitude estimation, any agent can estimate its relative position in a framework similar to the one shown in Fig. \ref{fig:decentralized}. The relative position estimator can just be a simple nonlinear least squares algorithm, or a more complex filtering algorithm. In what follows, a centralized framework is considered to demonstrate the use of $n_t \geq 2$ two-tag agents and $n_s$ single-tag agents for relative positioning, and decentralization is reserved for future work.

The centralized relative position and attitude estimation problem for a swarm of $N = n_t + n_s$ agents considered herein involves estimating the state vector 
\begin{equation}
    \mbf{x}(t) = \left[ \begin{array}{c}
        \mbf{r}_a^{z_2 z_1} (t)  \\
        \vdots \\
        \mbf{r}_a^{z_{N} z_1} (t) \\
        \mbf{v}_a^{z_2 z_1 } (t) \\
        \vdots \\
        \mbf{v}_a^{z_{N} z_1 } (t) \\
        \mbs{\phi}_1 (t) \\
        \vdots \\
        \mbs{\phi}_{N} (t)
    \end{array} \right] \in \mathbb{R}^{6(N-1) + 3N}, \label{eq:state_vec}
\end{equation}
where $\mbf{v}_a^{z_i z_1} \in \mathbb{R}^3$ is the velocity of the IMU of Agent $i$ relative to the IMU of the reference agent Agent $1$ with respect to $\mathcal{F}_a$, resolved in $\mathcal{F}_a$, and $\mbs{\phi}_i \in \mathbb{R}^3$ is the rotation vector associated with the DCM of agent $i$. Note that this approach involves the estimation of the attitude of the $n_s$ single-tag agents as well.

Let $\mbf{u}_{b_i}^{\text{acc}}, \mbf{u}_{b_i}^{\text{gyr}} \in \mathbb{R}^3$ denote the accelerometer and gyroscope readings of agent $i$, respectively. The process model of the relative states of agent $i$ is then modelled as 
\begin{align}
    \mbfdot{r}_a^{z_iz_1}(t) &= \mbf{v}_a^{z_iz_1}(t), \label{eq:process_first}\\
    \mbfdot{v}_a^{z_iz_1}(t) &= \mbf{C}_{ab_i}(t) \left( \mbf{u}_{b_i}^{\text{acc}}(t) + \mbf{w}_{b_i}^{\text{acc}}(t) \right) \\
    &\hspace{15pt} - \mbf{C}_{ab_1}(t) \left( \mbf{u}_{b_1}^{\text{acc}}(t) + \mbf{w}_{b_1}^{\text{acc}}(t) \right),
\end{align}
where $\mbf{w}_{b_i}^{\text{acc}} \in \mathbb{R}^3$ denotes the white Gaussian noise associated with the accelerometer measurement of the $i^\text{th}$ agent. The process model of Agent $i$'s attitude is modelled as 
\begin{equation}
    \mbfdot{C}_{ab_i}(t) = \mbf{C}_{ab_i}(t) \left( \mbf{u}_{b_i}^\text{gyr}(t) + \mbf{w}_{b_i}^\text{gyr}(t) \right)^\times, \label{eq:process_end}
\end{equation}
where $\mbf{w}_{b_i}^{\text{gyr}} \in \mathbb{R}^3$ denotes the white Gaussian noise associated with the gyroscope measurement of the $i^\text{th}$ agent, and as before, the $(\cdot)^\times$ denotes the skew-symmetric cross product matrix operator in $\mathbb{R}^3$

To estimate the state vector \eqref{eq:state_vec} in a centralized framework, all agents communicate their measurements to a master agent, which could be any of the $N$ agents, along with noisy range measurements between the $i^\text{th}$ tag of Agent $k$ and the $j^\text{th}$ tag of Agent $\ell$ of the form
\begin{align}
    y_{p_{k,i}p_{\ell,j}} &=  \norm{ \left( \mbf{r}_{a}^{p_{\ell,j} z_\ell} - \mbf{r}_{a}^{p_{k,i} z_k} \right)} + \nu_{p_{k,i}p_{\ell,j}}, \label{eq:measurement_first}
\end{align}
where $\nu_{p_{k,i}p_{\ell,j}} \in \mathbb{R}$ represents the white Gaussian noise associated with the range measurement $y_{p_{k,i}p_{\ell,j}}$. In addition to the range measurements, accelerometer aiding \cite{farrell2008} and magnetometer measurements are implemented to correct attitude drift.

The process models \eqref{eq:process_first}-\eqref{eq:process_end} are discretized using a forward Euler discretization scheme, and the process models and measurement models are linearized using a first-order Taylor series approximation. A centralized multiplicative extended Kalman filter (MEKF) in the spirit of \cite{farrell2008} is then designed and evaluated in simulation in Section \ref{sec:sim}, and in an experiment in Section~\ref{sec:exp}.

\section{Simulation Results} \label{sec:sim}

\begin{table}
\renewcommand{\arraystretch}{1.1}
\caption{Simulation parameters used in the Monte Carlo trials.}
\label{tab:sim_params}
\centering
\begin{tabular}{|c|c|}
\hline
\bfseries Specification & \bfseries Value\\
\hline
Accelerometer std. dev. (m/s$^2$) & 0.026 \\
Gyroscope std. dev. (rad/s) & 0.0025 \\
Magnetometer std. dev. ($\mu$F) & 0.85 \\
UWB std. dev. (m) & 0.1 \\
IMU rate (Hz) & 100 \\
UWB rate (Hz) & 20 \\
No. of UWB freq. channels & 3 \\
Initial relative position std dev. (m) & 0.45 \\
Initial relative velocity std dev. (m/s) & 0.45 \\
Initial attitude std dev. (rad) & 0.1 \\
\hline
\end{tabular}
\end{table}

\begin{figure*}
    \centering
    \includegraphics[trim=4cm 0.8cm 3.5cm 0cm, clip=true, width=\linewidth]{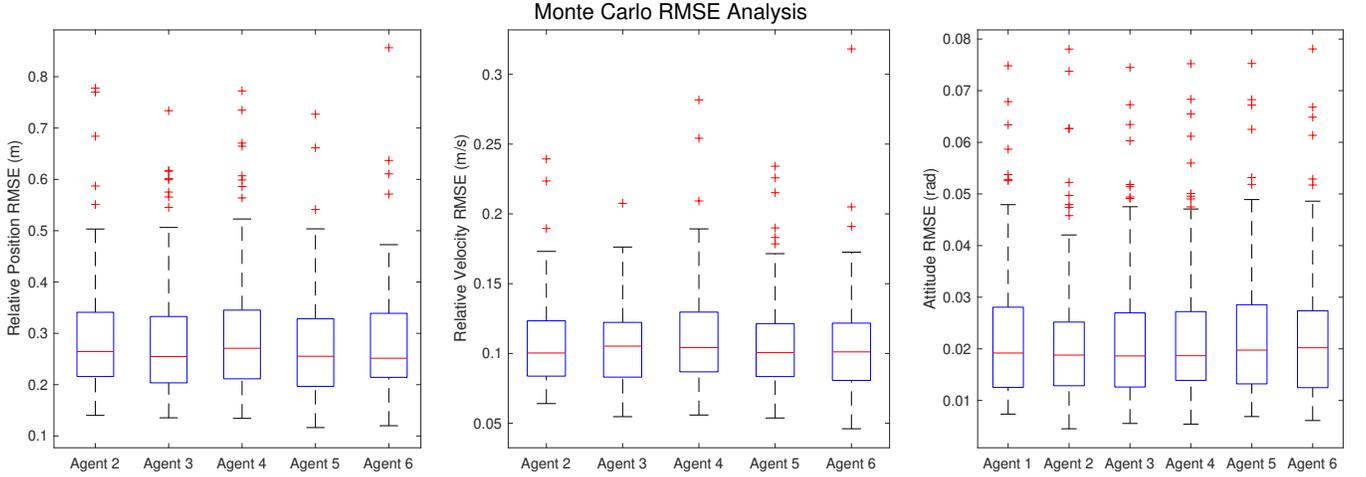}
    \caption{A box plot for the RMSE on 100 Monte Carlo trials. Although there are 3 to 8 outliers in the relative position estimates of the agents, the estimator still achieves an RMSE below 1 m accuracy for all runs, below 0.35 m/s for the velocity estimates, and below 0.08 rad for the attitude estimates.}
    \label{fig:RMSE}
\end{figure*}

Consider 6 fully-connected aerial robots equipped with an IMU and ultra-wideband (UWB) ranging tags, where agents 1, 2, and 3 are two-tag agents, and agents 4, 5, and 6 are single-tag agents. Let the 3 known relative tag positions be 
\begin{align*}
    \mbf{r}_{b_1}^{p_{1,2} p_{1,1}} = \left[\hspace{-3pt} \begin{array}{c}
        0.3  \\
        0 \\
        0
    \end{array} \hspace{-3pt}\right]\hspace{-1pt}, \hspace{2pt} \mbf{r}_{b_2}^{p_{2,2} p_{2,1}} = \left[\hspace{-3pt} \begin{array}{c}
        0  \\
        0.3\\
        0
    \end{array} \hspace{-3pt}\right]\hspace{-1pt}, \hspace{2pt} \mbf{r}_{b_3}^{p_{3,2} p_{3,1}} = \left[\hspace{-3pt} \begin{array}{c}
        0  \\
        0 \\
        0.3
    \end{array} \hspace{-3pt}\right]\hspace{-1pt},
\end{align*}
where all values are given in metres. Additionally, let Agent $1$ be the elected \emph{reference agent}, where a reference agent is specified similarly to a reference tag in Section \ref{sec:suff_cond}. In this section, the developed framework is evaluated by fusing the range measurements with an IMU using an MEKF to find the position of agents 2-6 relative to the reference agent as they move in 3-dimensional space. The centralized state estimator discussed in Section \ref{subsec:models} is assumed to be on Agent 1. The simulation parameters are given in Table \ref{tab:sim_params}.

To assess the performance of the MEKF with the two-tag framework, 100 Monte Carlo trials with different initial conditions and noise realizations are performed, and the corresponding root-mean-squared-error (RMSE) on the relative position, relative velocity, and attitude states are shown in Fig. \ref{fig:RMSE}. When considering all 100 runs, an average RMSE of 0.2887 m, 0.1080 m/s, and 1.306$^\circ$ for the position, velocity, and attitude states respectively are achieved. A normalized estimation error squared (NEES) test \cite[Section 5.4]{barshalom2002} is performed as shown in Fig. \ref{fig:NEES} to verify the consistency of the estimator.

\begin{figure}
    \centering
    \includegraphics[trim=0.0cm 0.0cm 0.25cm 0cm, clip=true,width=0.9\columnwidth]{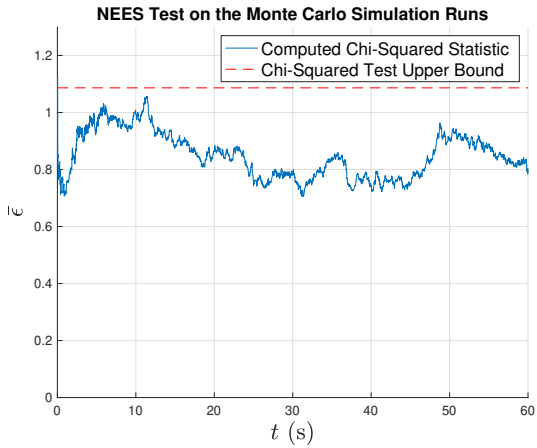}
    \caption{A plot representing the NEES test. The computed chi-squared statistic is below the upper bound, indicating the estimator is consistent.}
    \label{fig:NEES}
\end{figure}

\section{Experimental Results} \label{sec:exp}

\begin{figure}
	\centering
	\begin{minipage}{0.42\columnwidth}
		\centering
        \includegraphics[trim=0.5cm 0cm 0cm 0cm, clip=true,width=\textwidth, ]{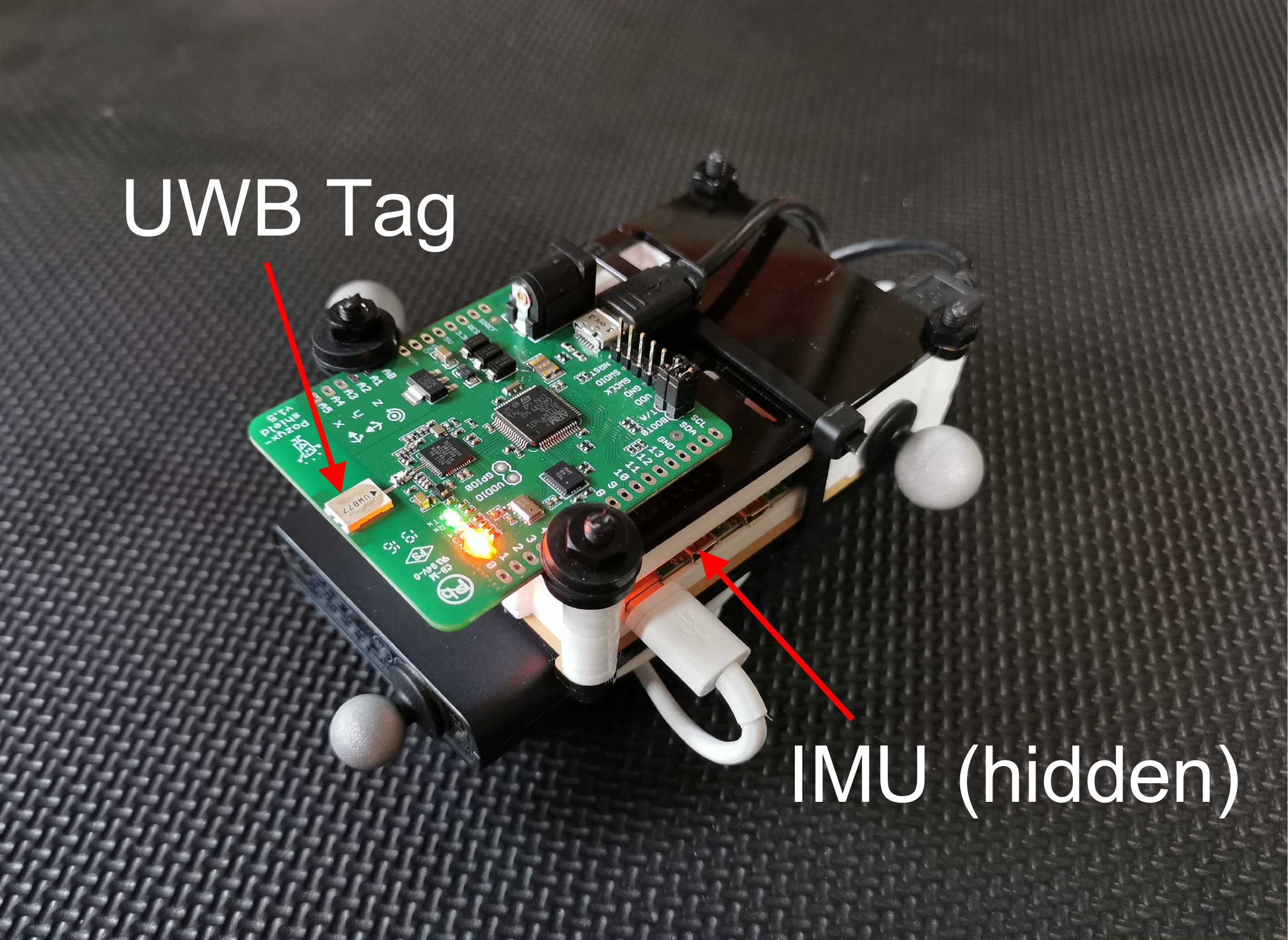}
	\end{minipage} \hspace{1pt}
	\begin{minipage}{0.54\columnwidth}
		\centering
        \includegraphics[width=\textwidth]{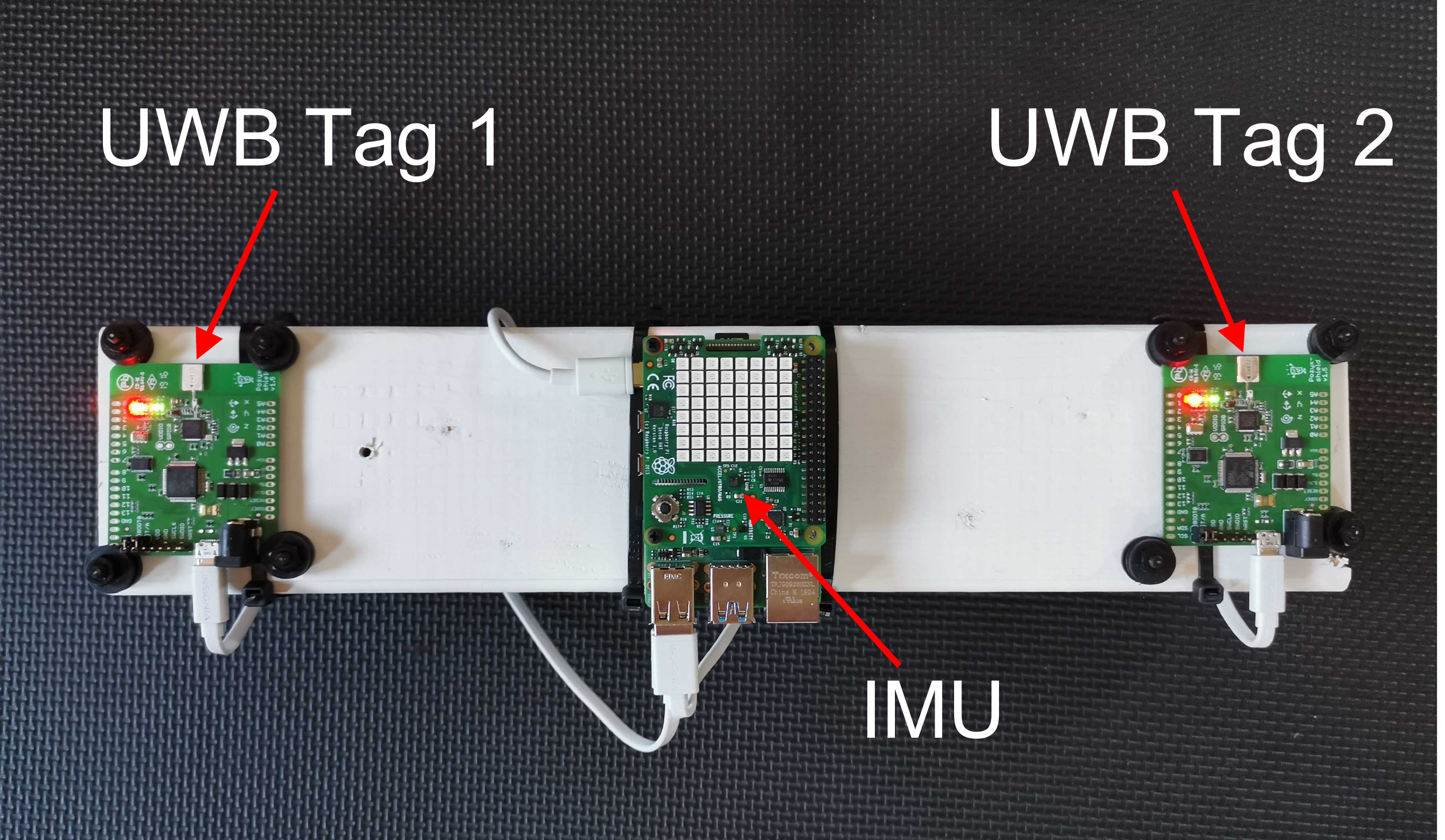}
	\end{minipage}
	\caption{The experimental set-up, showing a single-tag agent (left) and a two-tag agent (right).}
        \label{fig:exp_setup}
\end{figure}

Experimental data is collected for a set-up with two two-tag agents, and a single-tag agent, in a fully-connected structure. The prototypes of a single-tag agent and a two-tag agent are shown in Fig. \ref{fig:exp_setup}. The two two-tag agents are set to be Agent 1 and Agent 2, and Agent~1 is set to be the reference agent, with
\begin{align*}
    \mbf{r}_{b_1}^{p_{1,2} p_{1,1}} &= \left[ \begin{array}{c}
        -0.0067 \\
        0.3172 \\
        -0.0185
    \end{array} \right], \quad \mbf{r}_{b_1}^{p_{2,2} p_{2,1}} = \left[ \begin{array}{c}
        0.0043 \\
        0.3213 \\
        0.0224
    \end{array} \right],
\end{align*}
where all the values are in metres. The IMU data is collected at 240 Hz using a Raspberry Pi Sense HAT device, and Pozyx UWB Developer Tags are used for ranging. Only one frequency channel is used at a communication rate of 16 Hz; therefore, each range measurement is collected at a frequency of only 2 Hz. Additionally, ground truth position and attitude measurements are collected at 120 Hz using an OptiTrack optical motion capture system.

The data is collected by moving all three agents indoors in random 3-dimensional rotational and translational motion, in a volume of approximately 5 m $\times$ 4 m $\times$ 2 m. The magnetometers are affected both by perturbations from the surroundings and the other agents, making estimation, especially attitude estimation, more difficult. Despite that, and with such a low ranging frequency and an inexpensive IMU, a relative position RMSE of 0.4890 m is achieved for Agent 2, and 0.42813 m for Agent 3, with the error and $\pm 3 \sigma$ confidence bounds plotted in Fig. \ref{fig:SR_EX}. This asserts the potential of the two-tag framework on indoor self-localization without the need for expensive hardware or computationally expensive algorithms, such as visual odometry.

On flying quadcopters, vibrations affecting the IMU readings might result in worse relative position estimates. Additionally, a larger volume might degrade the performance of the algorithm as the measurements to the two tags from another agent become less geometrically distinct. However, by implementing more than just 3 agents and/or by increasing the distance between the two tags of the two-tag agents, the performance of the estimator improves and might compensate for worse attitude estimates or larger distance between the agents.

\begin{figure}
	\centering
	\begin{minipage}{\columnwidth}
		\centering
		\includegraphics[trim=1cm 0cm 1cm 0cm, clip=true,width=0.9\textwidth]{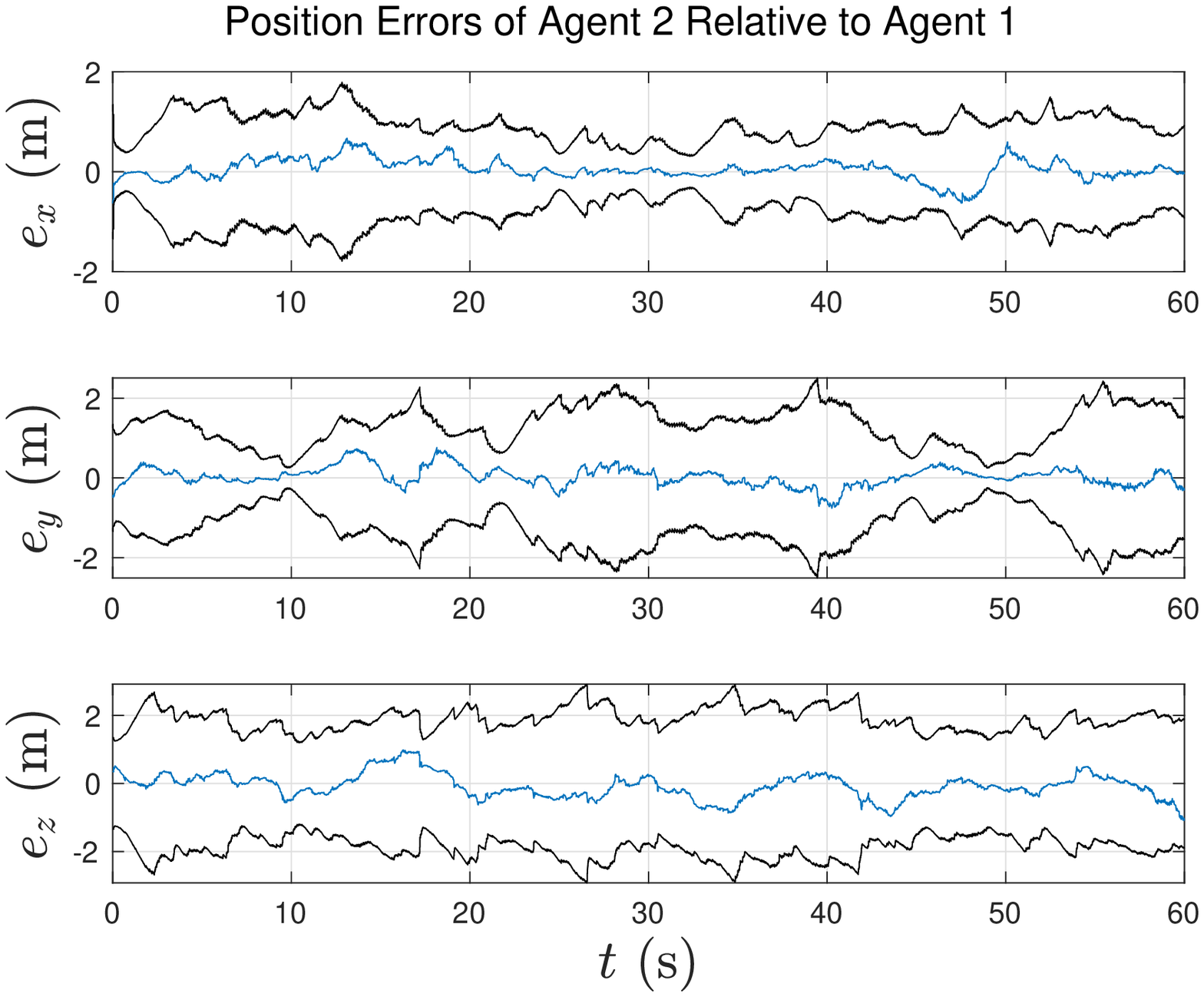}
	\end{minipage}
	\begin{minipage}{\columnwidth}
		\centering
		\includegraphics[trim=1cm 0cm 1cm 0cm, clip=true,width=0.9\textwidth]{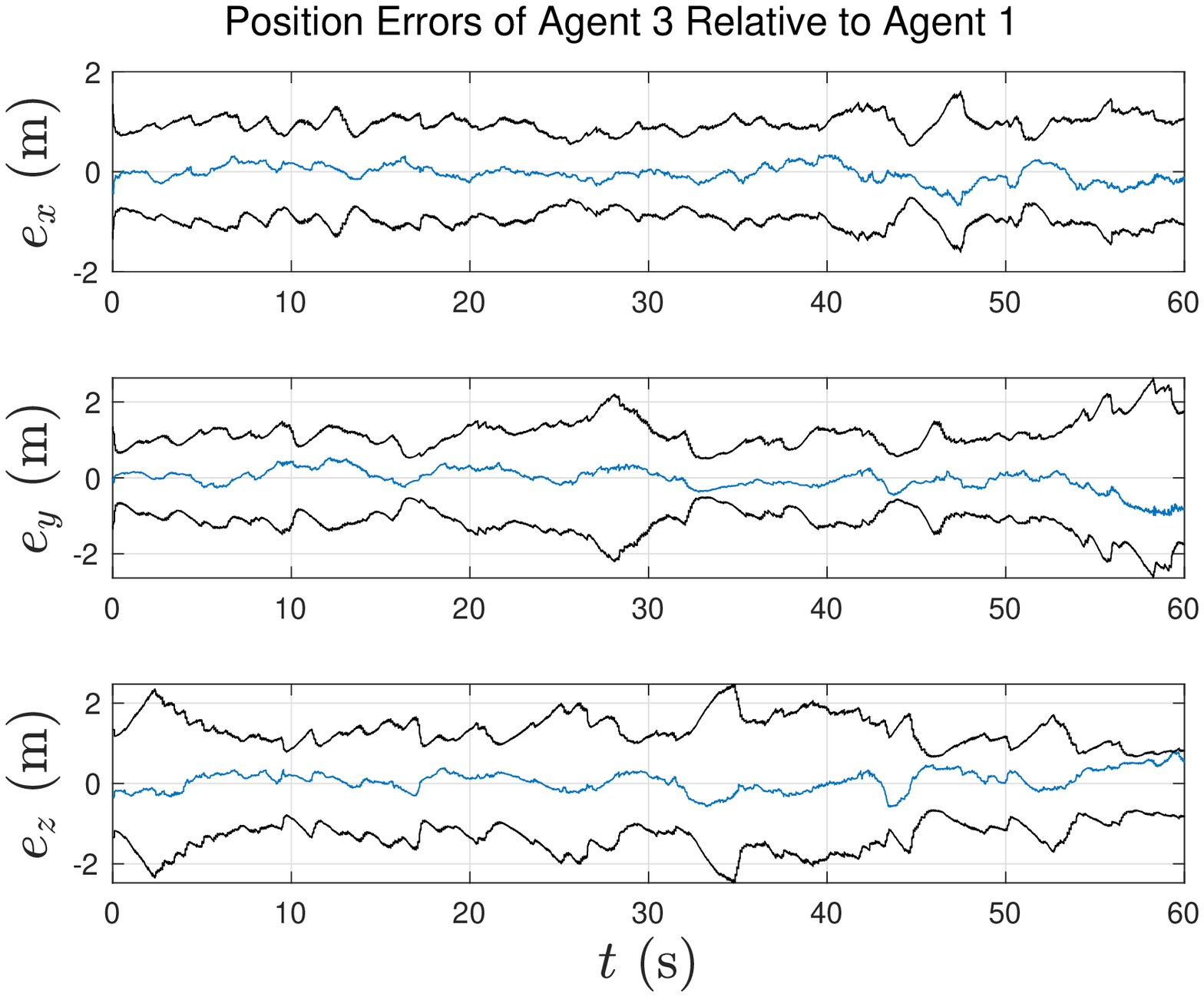}
	\end{minipage}
	\caption{The performance of the relative position estimator on experimental data for both the two-tag Agent 2 and the single-tag Agent 3, where the black lines represent the $\pm3\sigma$ bound of the estimator.}
	\label{fig:SR_EX}
\end{figure}

\section{Conclusion and Future Work} \label{sec:conclusion}

In this letter, the problem of three-dimensional relative position estimation using range measurements is addressed. The first step involves deriving a sufficient condition such that the relative position states of the agents are instantaneously locally observable. Thereafter, a framework utilizing two-tag agents is developed, which exploits attitude information to satisfy the sufficient conditions for observability. Lastly, this framework is integrated with an IMU using an MEKF and is tested in simulation and in experiments. The results show that around 40-50 cm relative positioning accuracy is achievable, using just an IMU and range measurements with three agents equipped with inexpensive sensors. Future work includes evaluating the two-tag framework on many agents in a decentralized structure.

\ifCLASSOPTIONcaptionsoff
  \newpage
\fi



\bibliographystyle{IEEEtran}
\bibliography{IEEEabrv,IEEEbib}

\end{document}